\documentclass[preprint,review,12pt]{elsarticle}

\usepackage{amsmath,amsfonts}
\usepackage{algorithmic}
\usepackage[linesnumbered,ruled,vlined]{algorithm2e}
\usepackage{array}
\usepackage[caption=false,font=normalsize,labelfont=sf,textfont=sf]{subfig}
\usepackage{textcomp}
\usepackage{stfloats}
\usepackage{url}
\usepackage{verbatim}
\usepackage{graphicx}
\usepackage{natbib} 
\usepackage{multirow}
\usepackage{amssymb}
\usepackage[normalem]{ulem}
\useunder{\uline}{\ul}{}
\usepackage{enumitem}
\usepackage{xcolor}
\usepackage{amsthm}
\usepackage{subcaption}
\usepackage{adjustbox}
\usepackage{wrapfig}
\theoremstyle{definition}
\newtheorem{definition}{Definition}
\newtheorem{theorem}{Theorem}

\journal{Neural Networks}

\begin{document}

\begin{frontmatter}



\title{Enhancing Signed Graph Neural Networks through
Curriculum-Based Training}

\author[1]{Zeyu Zhang}  
\author[1]{Lu Li}  
\author[1]{Xingyu Ji}  
\author[2]{Kaiqi Zhao}  
\author[3]{Xiaofeng Zhu}  
\author[4]{Philip S. Yu} 
\author[1]{Jiawei Li}  
\author[1]{Maojun Wang} 

\affiliation[1]{organization={ the College of the Informatics, Huazhong Agricultural University}}

\affiliation[2]{organization={the School of Computer Science, University of Auckland}}

\affiliation[3]{organization={the School of Computer Science and Engineering, University of Electronic Science and Technology of China}}

\affiliation[4]{organization={the department of Computer Science, University of Illinois}}
                        


\begin{abstract}
Signed graphs are powerful models for representing complex relations with both positive and negative connections. Recently, Signed Graph Neural Networks (SGNNs) have emerged as potent tools for analyzing such graphs. To our knowledge, no prior research has been conducted on devising a training plan specifically for SGNNs. The prevailing training approach feeds samples (edges) to models in a random order, resulting in equal contributions from each sample during the training process, but fails to account for varying learning difficulties based on the graph’s structure. We contend that SGNNs can benefit from a curriculum that progresses from easy to difficult, similar to human learning. The main challenge is evaluating the difficulty of edges in a signed graph. We address this by theoretically analyzing the difficulty of SGNNs in learning adequate representations for edges in unbalanced cycles and propose a lightweight difficulty measurer. This forms the basis for our innovative \underline{C}urriculum representation learning framework for \underline{S}igned \underline{G}raphs, referred to as \textbf{CSG}. The process involves using the measurer to assign difficulty scores to training samples, adjusting their order using a scheduler and training the SGNN model accordingly.  
We empirically our approach on six real-world signed graph datasets. Our method demonstrates remarkable results, enhancing the accuracy of popular SGNN models by up to 23.7\% and showing a reduction of 8.4\% in standard deviation, enhancing model stability. Our implementation is available in PyTorch\footnote{https://github.com/Alex-Zeyu/CSG}.  
\end{abstract}


\begin{highlights}
\item We are pioneering research on training methods for signed graph neural networks.
\item Our work provides a theoretical analysis of the limitations in current SGNNs.
\item We introduce a tool to assess sample complexity.
\item We introduce a novel curriculum 
 learning approach for signed graphs(CSG).
\item The CSG method improves the performance and stability of backbone models.
\end{highlights}

\begin{keyword}
Graph Neural Networks \sep Signed Graph representation learning \sep Curriculum Learning


\end{keyword}

\end{frontmatter}


\section{Introduction}



Online social networks,  recommendation system, cryptocurrency platforms, and even genomic-phenotype association studies have led to a significant accumulation of graph datasets. 
To analyze these graph datasets, graph representation learning \cite{velivckovic2017graph, yan2024contrastive,kipf2016semi} methods have gained popularity, especially those based on graph neural networks (GNNs). GNNs use a message-passing mechanism to generate expressive representations of nodes by aggregating information along the edges. However, real-world edge relations between nodes are not limited to positive ties such as friendship, like, trust, and upregulation; they can also encompass negative ties like enmity,   
 dislike, mistrust, and downregulation, as shown in Figure \ref{fig:illustration}. For example, in social networks, users can be tagged as both `friends' and `foes' on platforms like Slashdot, a tech-related news website. In biological research, traits are influenced by gene expression regulation, which involves upregulation and downregulation. This scenario naturally lends itself to modeling as a {\em signed graph}, which includes both positive and negative edges. Nevertheless, the presence of negative edges complicates the standard message-passing mechanism, necessitating the development of new GNN models tailored to signed graphs — signed graph neural networks (SGNNs).

\begin{figure}[!t]
    \centering    \includegraphics[width=1.0\columnwidth]{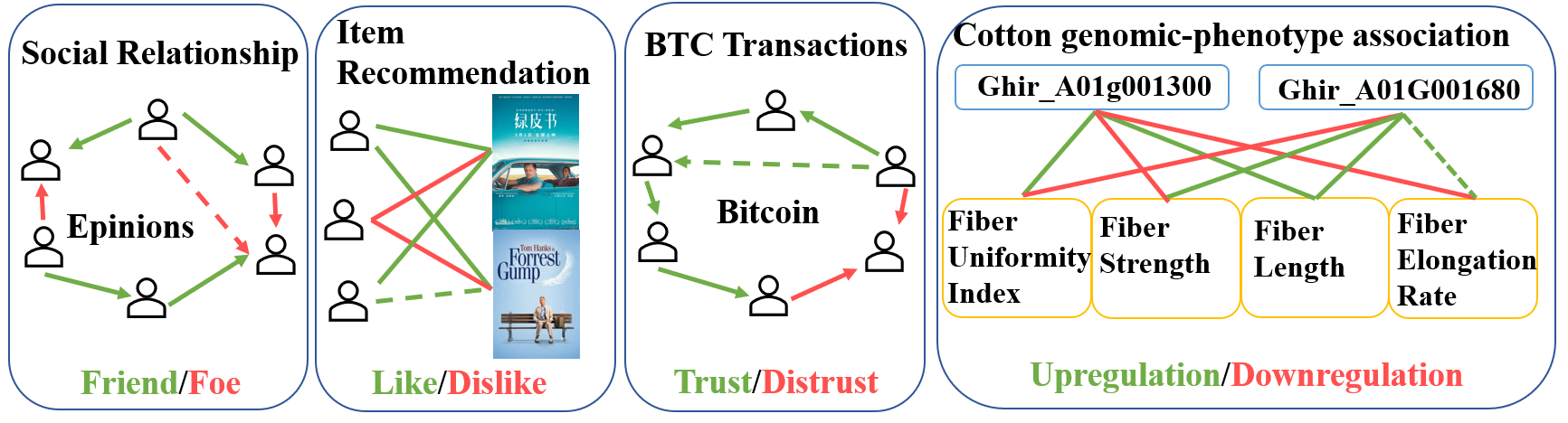}
    \caption{An illustration of signed graphs in real world.}
    \label{fig:illustration}
\end{figure}

While much effort has gone into developing new SGNN models \cite{derr2018signed,lee2020asine} for link sign prediction, research on their training methods is still lacking. Currently, SGNNs are trained by treating all edges equally and presenting them in random order. However, edges can have varying levels of learning difficulty. For example, Fig.\ref{fig:triad} shows four isomorphism types of undirected signed triangles. Intuitively, if node $v_i$ and node $v_j$ are connected by a positive edge, their positions in the embedding space should be made as close as possible, whereas if node $v_i$ and node $v_j$ are connected by a negative edge, their positions in the embedding space should be made as far apart as possible \cite{cygan2012sitting}. Nevertheless, in Fig.\ref{fig:triad}(c), node $v_i$ is connected to node $v_j$ by a negative edge, so in the embedding space, the distance between node $v_i$ and node $v_j$ should be as far as possible. However, node $v_i$ is connected to node $v_k$ and node $v_k$ is connected to node $v_j$, both with positive edges. Therefore, in the embedding space, node $v_i$ should be as close as possible to node $v_k$, and node $v_k$ should be as close as possible to node $v_j$. Consequently, the distance between node $v_i$ and node $v_j$ should be as close as possible. This contradiction makes it much harder for SGNNs to learn adequate  representations (see Def.\ref{def:adequate}) for these nodes from unbalanced triangles. To alleviate this situation, a direct approach is to reduce the impact of samples belonging to unbalanced cycles on the model. Extensive research has demonstrated that presenting training samples in a thoughtful sequence can yield substantial benefits for deep learning models \cite{bengio2009curriculum,wei2022clnode}. Specifically, initiating the training process with simpler examples and gradually introducing more complex ones has been shown to significantly enhance the models' performance. The methodology is recognized as \textit{Curriculum Learning}.

\begin{figure}[!t]
    \centering
    \includegraphics[width=0.8\columnwidth]{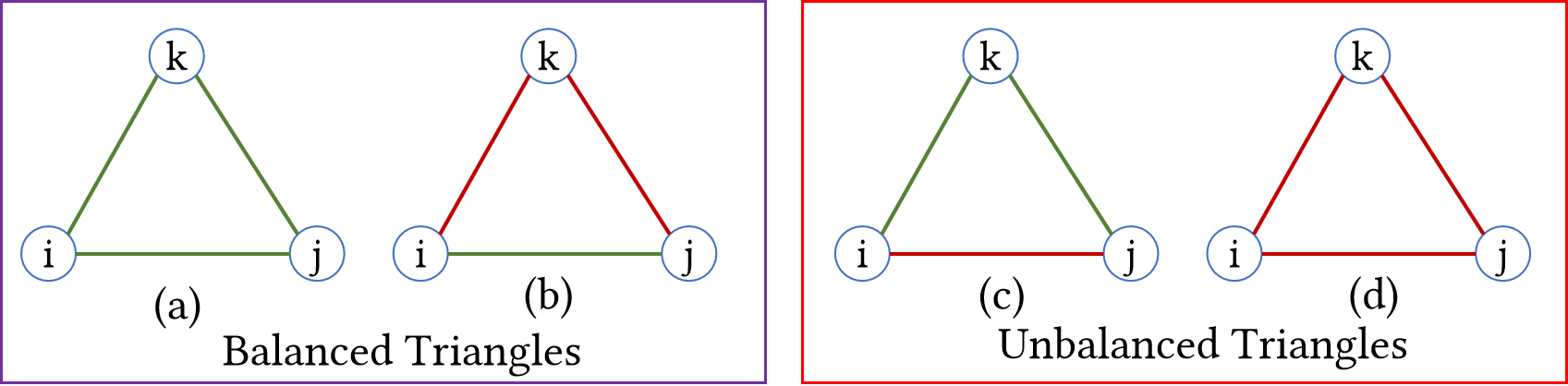}
    \caption{Balanced (unbalanced) triangles (3-cycles). Green and red lines represent positive and negative edges, resp.}
    \label{fig:triad}
\end{figure}

Curriculum learning is at the intersection between cognitive science and machine learning \cite{he2023boosting,fu2025gsscl}. Inspired by humans' learning habits, extensive research discovers that feeding the training samples in the ascending order of their hardness can benefit machine learning~\cite{xu2020curriculum}. Intuitively speaking, curriculum learning strategically mitigates the adverse effects of challenging or noisy samples during the initial training phases and gradually expose the model to increasingly complex and informative examples, thereby aiding the model in building a more robust and generalized understanding of the task at hand. CurGraph~\cite{wang2021curgraph} is the first to introduce curriculum learning to GNNs. However, it is designed for unsigned graph classifications. To the best of our knowledge, curriculum learning for SGNNs with link sign prediction as main downstream task remains unexplored.  

The main challenge when designing a curriculum learning method for SGNNs lies in how to evaluate the difficulty of training samples (i.e., edges). Graph-level classification models, e.g., CurGraph which claims the difficulty of samples (i.e., graphs) depends on the complexity of graphs (e.g., the number of nodes or edges in graphs). However, for the primary task of signed graph analysis, which is link sign prediction, the training samples (edges) are not independent, so it is not trivial to measure the difficulty of these samples. Alternative approaches frequently make use of label information \cite{wei2022clnode} and node features \cite{chu2021cuco} to differentiate between the levels of complexity in training samples. However, such data is absent in current real-world signed graphs. In this paper, we first theoretically analyze the learning difficulty of edges. We prove that current SGNNs cannot learn adequate representations for edges belonging to unbalanced cycles. Based on this conclusion, we design a lightweight difficulty assessment function to score the difficulty of edges. We encapsulate this idea in a new SGNN framework, called CSG (\underline{C}urriculum representation learning for \underline{S}igned \underline{G}raphs). CSG sorts the training samples by their difficulty scores and employs a training scheduler that continuously appends a set of harder training samples to the learner in each epoch. It is worth noting that postponing the training of hard samples will reduce the importance of hard examples in the training process \cite{wei2022clnode} but cannot enable SGNN models to surpass their current limitations, namely, learning adequate representations from unbalanced triangles. This facilitates SGNN models in learning more effective representations from easy edges, ultimately enhancing the overall predictive performance for both easy and hard edges see Table \ref{tab:easy}.

 To evaluate the effectiveness of CSG, we perform extensive experiments on six real-world datasets. We verify that our proposed method CSG can improve the link sign prediction performance of the backbone models by up to 23.7\% (SGCN \cite{derr2018signed} model, Slashdot dataset) in terms of AUC and can enhance the stability of models, achieving a standard deviation reduction of up to 8.4\% on AUC (SGCN \cite{derr2018signed} model, WikiRfa) (see Table \ref{tab:auc}). In addition, we also verify that on more incomplete graphs, say 40\% - 70\% training ratio, CSG can still enhance the performance of backbone models (see Table \ref{tab:incomplete}). These experimental results demonstrate the effectiveness of CSG. One limitation of our method is that we only consider unbalanced triangles (3-cycle). This is due to the high sparsity commonly observed in the current real-world signed graph datasets (see Table \ref{tab:datasets}). Therefore, the number of unbalanced 4-cycles, 5-cycles, and even 6-cycles is relatively much smaller (see Table \ref{tab:n-cycle}). To ensure algorithmic simplicity and efficiency, this paper only consider 3-cycles (i.e., triangles). Overall, our contributions are summarized as follows:


\begin{itemize}
    \item We are pioneering research in the field of training methods for signed graph neural networks (SGNNs).
    \item We implement curriculum learning in the training of SGNNs. Our work involves a thorough theoretical analysis of the inherent limitations within current SGNNs. Utilizing these insights, we introduce a lightweight difficulty assessment tool capable of assigning complexity scores to samples (e.g., edges).
    \item We introduce an innovative curriculum representation learning approach tailored for signed graphs (CSG).
    \item We evaluate CSG on six real-world signed graph datasets using five backbone SGNN models. The results highlight CSG's effectiveness as a curriculum learning framework, improving both accuracy and stability across these models.
\end{itemize}

\section{Related Work}

\subsection{Signed Graph Representation Learning}
Due to social media's popularity, signed graphs are now widespread, drawing significant attention to network representation \cite{zhang2023contrastive,zhang2023rsgnn, he2024signed,lin2022status,ko2023spectral}. Existing research mainly focuses on \textit{link sign prediction}, despite other tasks like node classification \cite{tang2016node}, node ranking \cite{jung2016personalized}, community detection \cite{bonchi2019discovering} and genomic-phenotype association prediction \cite{pan2024csgdn}. Some signed graph embedding methods, such as SNE \cite{yuan2017sne}, SIDE \cite{kim2018side}, SGDN \cite{jung2020signed} are based on random walks and linear probabilistic methods. In recent years, neural networks have been applied to signed graph representation learning. SGCN \cite{derr2018signed}, the first SGNN that generalizes GCN \cite{kipf2016semi} to signed graphs, uses balance theory to determine the positive and negative relationships between nodes that are multi-hop apart. Another important GCN-based work is GS-GNN \cite{liu2021signed} which alleviates the assumption of balance theory and generally assumes nodes can be divided into multiple groups. In addition, other main SGNN models such as SiGAT \cite{huang2019signed}, SNEA \cite{li2020learning}, SDGNN \cite{huang2021sdgnn}, and SGCL \cite{shu2021sgcl} are based on GAT \cite{velivckovic2017graph}. These works mainly focus on developing more advanced SGNN models. Our work is orthogonal to these works in that we propose a new training strategy to enhance SGNNs by learning from an easy-to-difficult curriculum.  

\subsection{Curriculum Learning}
Curriculum Learning proposes the idea of training models in an easy-to-difficult fashion inspired by cognitive science \cite{rohde1999language}, which implies that one can improve the performance of machine learning models by feeding the training samples from easy to difficult. In recent years, Curriculum Learning has been employed in Computer Vision \cite{basu2022surpassing} and Natural Language Processing \cite{agrawal2022imitation}, which commonly follow the similar steps, i.e., 1) evaluating the difficulty of training samples, 2) schedule the training process based on the difficulties of training samples. CurGraph \cite{wang2021curgraph} is the first to develop a curriculum learning approach for graph classification, which uses the infograph method to obtain graph embeddings and a neural density estimator to model embedding distributions which is used to calculate the difficulty scores of graphs based on the intra-class and inter-class distributions of their embeddings. CLNode \cite{wei2022clnode} applies curriculum learning to node classification. CuCo \cite{chu2021cuco} applies curriculum learning to graph contrastive learning, which can adjust the training order of negative samples from easy to hard. In general, curriculum learning study for GNNs is still in its infants. To the best of our knowledge, there is no attempt to apply curriculum learning to Signed Graph Neural Networks. One essential challenge in designing the curriculum learning method is how to measure the difficulty of training samples (i.e., edges). The aforementioned methods often utilize label information \cite{wei2022clnode} and node feature \cite{chu2021cuco} to distinguish the difficulty levels of training samples. \textit{Such information is not available in existing real-world signed graphs}. Assessing the difficulty of training samples from signed graph structures remains an open question.

\section{Notations}
A  {\em signed graph} is $\mathcal{G}=(\mathcal{V}, \mathcal{E})$ where $\mathcal{V}=\{v_1, \cdots, v_{|\mathcal{V}|} \}$ denotes the set of nodes and $\mathcal{E} = \mathcal{E}^+ \cup \mathcal{E}^-$ denote the set of edges with positive and negative signs. The sign graph can  be represented by a signed \textit{adjacency matrix} $A \in \mathbb{R}^{|\mathcal{V}|\times |\mathcal{V}|}$ with entry $A_{ij} >0$ (or $<0$) if a positive (or negative) edge between node $v_i$ and $v_j$, and $A_{ij} =0$ denotes no edge between $v_i$ and $v_j$. Note that in real-world signed graph datasets, nodes usually lack features, unlike unsigned graph dataset which typically contains a feature vector $x_i$ for each node $v_i$. $\mathcal{N}_{i}^{+}=\{v_j\mid A_{ij}>0\}$ and $\mathcal{N}_{i}^{-}=\{v_j\mid A_{ij}<0\}$ denote the \textit{positive} and \textit{negative} neighbors of node $v_i$. $\mathcal{N}_{i} =  \mathcal{N}^+_{i} \cup \mathcal{N}^-_{i}$ denotes the neighbor set of node $v_i$. $\mathcal{O}_n$ denotes the set of n-cycles in the signed graph, e.g.,  $\mathcal{O}_3$ represents the set of {\em triangles} (3-cycles) in the signed graph. $\mathcal{O}_n^{+}(\mathcal{O}_n^{-})$ denotes the balanced (unbalanced) n-cycle set. A balanced (unbalanced) n-cycle with $n$ nodes has an even (odd) number of negative edges, e.g., A 4-cycle $\{v_i,v_j,v_k,v_l\}$ is called {\em balanced} ({\em unbalanced}) if $A_{ij}A_{jk}A_{ik}A_{kl}>0$ ($A_{ij}A_{jk}A_{ik}A_{kl}<0$). The main notations are shown in Table \ref{tab:notations}.

\begin{table}[]
\scriptsize
\centering
\caption{Key notations utilized in the paper}
\begin{tabular}{ll}
\hline
Notations                                   & Descriptions                          \\ \hline
$\mathcal{G}$                               & An undirected Signed Graph            \\
$\mathcal{V}$                               & Node set                              \\
$\mathcal{E}$                               & Edge set                              \\
$A$                                         & Adjacency matrix of $\mathcal{G}$     \\
$H$                                         & Output embedding matrix of nodes      \\
$\mathcal{E}^{+}(\mathcal{E}^{-})$          & Positive (negative) edge set          \\
$\mathcal{N}_{i}$                           & Neighbors of node $v_i$               \\
$\mathcal{N}_{i}^{+}(\mathcal{N}_{i}^{-})$  & Positive (negative) neighbors of node $v_i$ \\
$\mathcal{O}_n$                             & n-cycle set                               \\
$\mathcal{O}_n^{+}(\mathcal{O}_n^{-})$      & Balanced (unbalanced) n-cycle set \\ \hline
\end{tabular}
\label{tab:notations}
\end{table}

The most general assumption for signed graph embedding suggests a node should bear a higher resemblance with those neighbors who are connected with positive edges than those who are connected with negative edges (from extended structural balance theory \cite{cygan2012sitting}). The primary objective of an SGNN is to acquire an \textit{embedding function} $f_{\theta}\colon \mathcal{V} \rightarrow H$ that maps nodes within a signed graph onto a latent vector space $H$. This function aims to ensure that $f_{\theta}(v_i)$ and $f_{\theta}(v_j)$ are proximate if $e_{ij}\in \mathcal{E}^+$, and distant if $e_{ij}\in \mathcal{E}^-$. Moreover, we choose \textit{link sign prediction} as the downstream task of SGNN. This task is to infer the sign of $e_{ij}$ (i.e., $A_{ij}$) when provided with nodes $v_i$ and $v_j$ \cite{leskovec2010predicting}.

\section{Methodology}
In this section, we describe our CSG framework as shown in Figure \ref{fig:framework}. First, Through extensive theoretical analysis, we establish that SGNNs struggle to learn adequate representations for edges in \textit{unbalanced cycles}, making these edges a significant challenge for the model. As a result, we create a lightweight difficulty measurer function where edges belonging to unbalanced triangles will be assigned higher difficult scores. Subsequently, we introduce a training scheduler to initially train models using `easy' edges and gradually incorporated `hard' edges into the training process. It is worth highlighting that the curriculum learning approach does not facilitate SGNN models in learning adequate representations for `hard' edges. It solely downplay the training weight of the `hard' samples by not presenting them in the early training process.

\begin{figure*}
    \centering
    \includegraphics[width=1.0\columnwidth]{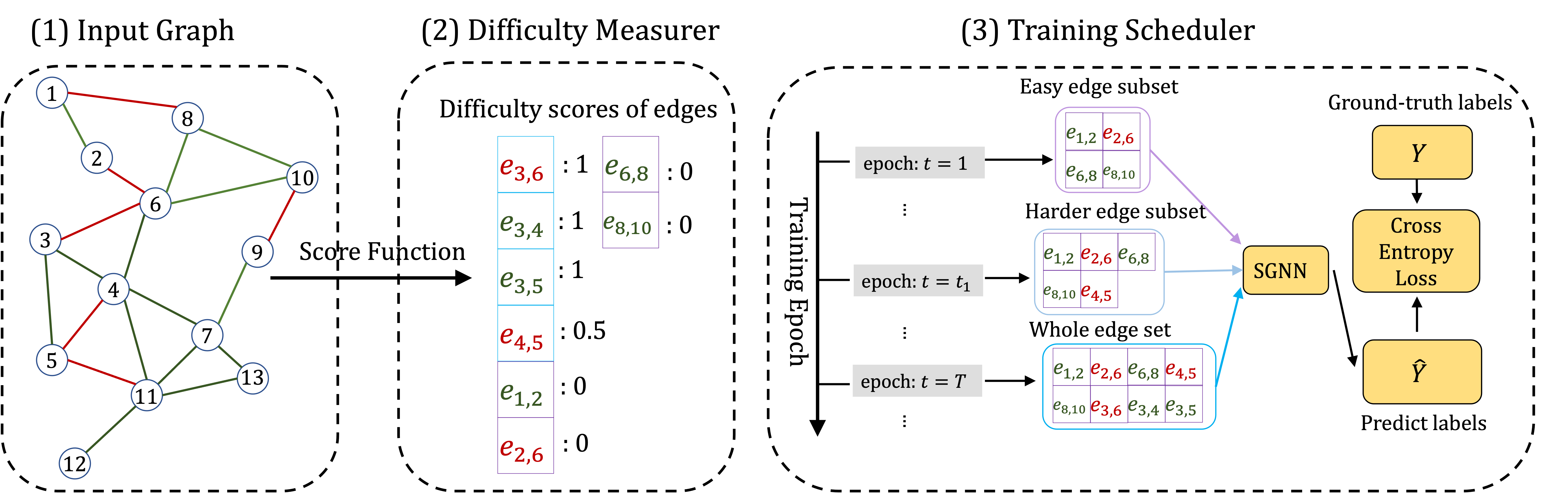}
    \caption{Overall framework of the proposed CSG. (1) input signed graph where green and red lines represent positive and negative edges, resp. (2) triangle-based difficulty measurer function where edges belonging to unbalanced triangles will be assigned higher difficult scores. (3) training scheduler where samples (edges) will be used to feed into the backbone SGNN models according to an easy-to-difficult curriculum.}
    \label{fig:framework}
\end{figure*}

\subsection{Theoretical Analysis}\label{sec:theory}
The key challenge in designing a curriculum learning-based training method for SGNNs is effectively identifying the difficulty of training samples (i.e., edges). We prove that learning adequate representations from unbalanced cycles poses a significant challenge for SGNNs. We first give the definition of \textit{Adequate representations of nodes}.
\begin{definition}[Adequate representations of nodes]\label{def:adequate}
Given a signed graph $\mathcal{G} = (\mathcal{V}, \mathcal{E})$, a SGNN model $f_{\theta}\colon \mathcal{V}\rightarrow H$ and any non-negative distance metric $dist\colon H\times H\rightarrow \mathbb{R}^+$, we call $H_i=f_{\theta}(v_i)$ an adequate representation of any node $v_i\in\mathcal{V}$ if the following conditions all satisfy:
\begin{itemize}
    \item[(a)] There exist $\epsilon>0$ such that for any $v_j \in \mathcal{N}_i^{-}$ and $H_j=f_{\theta}(v_j)$, $dist(H_i, H_j)>\epsilon$;
    \item[(b)] For any $v_j\in \mathcal{N}_i^{+}$, $v_k\in \mathcal{N}_i^{-}$ and $H_j=f_{\theta}(v_j)$, $H_k=f_{\theta}(v_k)$, $dist(H_i, H_j)< dist(H_i, H_k)$,
\end{itemize}
\end{definition}
An intuitive interpretation of Definition \ref{def:adequate} is that nodes linked by negative edges should be distant in embedding space, exceeding a certain positive threshold $\epsilon$ (Condition a), while nodes linked by positive edges should be closer in embedding space than those linked by negative edges (Condition b). We define \textit{Adequate representations of edges} based on this. 
\begin{definition}[Adequate representations of edges]\label{def:adequate_edges}
Given a node set $\mathcal{V}_{imp.} \in \mathcal{V}$ (imp. refers to improper) contains nodes with improper representations. The representation of edge $e_{ij}$ is $H_{ij} = [H_i,H_j]$, where $[,]$ is concatenation operator. We call $H_{ij}$ an adequate representation of any edge $e_{ij}$, if $v_i \notin \mathcal{V}_{imp.}$ \&\& $v_j \notin \mathcal{V}_{imp.}$
\end{definition}
We now give a concise introduction to the aggregation mechanism for SGNNs. Essentially, mainstream SGNN models such as SGCN \cite{derr2018signed} and SNEA \cite{li2020learning} adopt the following mechanism. 

The node $v_i$'s representation at layer $\ell$ is defined as:
\[
\scriptsize
h_i^{(\ell)} = [ h_i^{pos(\ell)}, h_i^{neg(\ell)} ]
\]
where $h_i^{pos(\ell)}$ and $ h_i^{neg(\ell)}$ respectively denote the positive and negative representation vectors of node $v_i \in \mathcal{V}$ at the $\ell$th layer and $[,]$ denotes a concatenation operation. The updating process at layer $\ell=1$ is written as:
\begin{equation}
\label{eq:SGNN_1}
\scriptsize
\begin{aligned}
    a_i^{pos(1)} &= \textrm{AGGREGATE}^{(1)} \left( \left\{h_j^{(0)}: v_j \in \mathcal{N}_i^{+} \right\} \right), 
    h_i^{pos(1)} &= \textrm{COMBINE}^{(1)}\left(h_i^{(0)},a_i^{pos(\ell)} \right) \\
    a_i^{neg(1)} &= \textrm{AGGREGATE}^{(1)} \left( \left\{h_j^{(0)}: v_j \in \mathcal{N}_i^-\right\} \right), 
    h_i^{neg(1)} &= \textrm{COMBINE}^{(1)}\left(h_i^{(0)},a_i^{neg(\ell)} \right) \\
\end{aligned}
\end{equation}
And for  $\ell>1$ layer, we have:
\begin{equation}
\label{eq:SGNN_2}
\scriptsize
\begin{aligned}
    a_i^{pos(\ell)} &= \textrm{AGGREGATE}^{(\ell)} \Biggl( \left\{h_j^{pos(\ell-1)}: v_j \in \mathcal{N}_i^+ \right\},
    &\left\{h_j^{neg(\ell-1)}: v_j \in \mathcal{N}_i^- \right\} \Biggl) \\
    h_i^{pos(\ell)} &= \textrm{COMBINE}^{(\ell)}\left(h_i^{pos(\ell-1)},a_i^{pos(\ell)} \right) \\
    a_i^{neg(\ell)} &= \textrm{AGGREGATE}^{(\ell)} \Biggl( \left\{h_j^{neg(\ell-1)}: v_j \in \mathcal{N}_i^+ \right\}, 
    &\left\{h_j^{pos(\ell-1)}: v_j \in \mathcal{N}_i^- \right\} \Biggl)\\
    h_i^{neg(\ell)} &= \textrm{COMBINE}^{(\ell)}\left(h_i^{neg(\ell-1)},a_i^{neg(\ell)} \right), \\
\end{aligned}
\end{equation}

Unlike GNNs, SGNNs handle positive and negative edges using a two-part representation and a more intricate aggregation scheme. For instance, when $\ell>1$, the positive part of the representation for node $v_i$ may aggregate information from the positive representation of its positive neighbors and the negative representation of its negative neighbors. In the upcoming discussion, we'll show that nodes in signed graphs with \textit{isomorphic ego trees} will have shared representations, building on SGNN's message-passing mechanism.



\begin{definition}[Signed graph isomorphism]\label{Signed Graph Isomorphism}
Two signed graphs $\mathcal{G}_1$ and $\mathcal{G}_2$ are {\em isomorphic}, denoted by $\mathcal{G}_1 \cong \mathcal{G}_2$, if there exists a bijection $\psi\colon \mathcal{V}_{\mathcal{G}_1} \rightarrow \mathcal{V}_{\mathcal{G}_2}$ such that, for every pair of vertices $v_i, v_j \in \mathcal{V}_{\mathcal{G}_1}$, $e_{ij} \in \mathcal{E}_1$, if and only if $e_{\psi(v_{i}),\psi(v_j)} \in \mathcal{E}_2$ and $\sigma(e_{ij}) = \sigma(e_{\psi(v_{i}),\psi(v_j)})$.
\end{definition}

We further define a node's balanced and unbalanced reach set, following similar notions in \cite{derr2018signed}.
\begin{definition}[Balanced / Unbalanced reach set]
\label{Relation node set}
For a node $v_i$, its {\em $\ell$-balanced (unbalanced) reach set $\mathcal{B}_i(\ell)$ ($\mathcal{U}_i(\ell))$) } is defined as a set of nodes with even (odd) negative edges along a path that connects $v_i$, where $\ell$ refers to the length of this path. 
The balanced (unbalanced) reach set extends positive (negative) neighbors from one-hop to multi-hop paths. In particular, the {\em balanced reach set} $\mathcal{B}_i(\ell)$ and the {\em unbalanced reach set} $\mathcal{U}_i(\ell))$ of a node $v_i$ with path length $\ell =1$ are defined as:
\begin{equation}
\scriptsize
\begin{aligned}
&\mathcal{B}_{i}(\ell)=\left\{v_{j} \mid v_{j} \in \mathcal{N}_{i}^{+}\right\} , 
&\mathcal{U}_{i}(\ell)=\left\{v_{j} \mid v_{j} \in \mathcal{N}_{i}^{-}\right\}
\end{aligned}
\end{equation}
For the path length $\ell >1$:
\begin{equation}
\scriptsize
\begin{aligned}
\mathcal{B}_{i}(\ell)=&\left\{v_{j} \mid v_{k} \in \mathcal{B}_{i}(\ell-1) \text { and } v_{j} \in \mathcal{N}_{k}^{+}\right\}
 \cup &\left\{v_{j} \mid v_{k} \in \mathcal{U}_{i}(\ell-1) \text { and } v_{j} \in \mathcal{N}_{k}^{-}\right\} \\
\mathcal{U}_{i}(\ell)=&\left\{v_{j} \mid v_{k} \in \mathcal{U}_{i}(\ell-1) \text { and } v_{j} \in \mathcal{N}_{k}^{+}\right\} 
 \cup &\left\{v_{j} \mid v_{k} \in \mathcal{B}_{i}(\ell-1) \text { and } v_{j} \in \mathcal{N}_{k}^{-}\right\}.
\end{aligned}
\end{equation}
\end{definition}

{\em Weisfeiler-Lehman (WL) graph isomorphism test} \cite{weisfeiler1968reduction} is a powerful tool to check if two unsigned graphs are isomorphic. A WL test recursively collects the connectivity information of the graph and maps it to the feature space. If two graphs are isomorphic, they will be mapped to the same element in the feature space. A multiset generalizes a set by allowing multiple instances for its elements. During the WL-test, a multiset is used to aggregate labels from neighbors of a node in an unsigned graph. More precisely, for a node $v_i$, in the $\ell$-th iteration, the node feature is the collection of node neighbors $\big \{ \left(X_i^{(\ell)},\{X_j^{(\ell)}\colon v_j \in \mathcal{N}_i \} \right) \big\}$ where $X_{i}^{(\ell)}$ denotes the feature (label) of node $v_i$.

We now extend WL test to signed graph: For a node $v_i$ in a signed graph, we use two multisets to aggregate information from $v_i$'s balanced reach set and unbalanced reach set separately. In this way, each node in a signed graph has two features $X_{i}(\mathcal{B})$ and $X_{i}(\mathcal{U})$ aside from the initial features.





\begin{definition}[Extended WL-test For Signed Graph]\label{Extended Wl-test}
Based on the message-passing mechanism of SGNNs, the process of extended WL-test for the signed graph can be defined as below. For the first-iteration update, i.e. $\ell=1$, the {\em WL node label} of a node $v_i$ is $(X_i^1 (\mathcal{B}),X_i^1(\mathcal{U}))$ where:
\begin{equation}
\scriptsize
\begin{aligned}
    X_{i}^{(1)}(\mathcal{B}) & = \varphi \bigg( \Big\{ \left( X_{i}^{(0)}, \{ X_{j}^{(0)}: v_j \in \mathcal{N}_{i}^{+} \} \right)  \Big\}   \bigg),  
    X_{i}^{(1)}(\mathcal{U}) & = \varphi \bigg( \Big\{ \left( X_{i}^{(0)}, \{ X_{j}^{(0)}: v_j \in \mathcal{N}_{i}^{-} \} \right)  \Big\}   \bigg)
\end{aligned}
\end{equation}
For $\ell>1$, the {\em WL node label} of $v_i$ is $(X_i^{(\ell)} (\mathcal{B}), X_i^{(\ell)}(\mathcal{U}))$ where:
\begin{equation}\label{signed label aggregation}\scriptsize
\begin{aligned}
    X_{i}^{(\ell)}(\mathcal{B}) & = \varphi \bigg( \Big\{ \bigl( X_{i}^{(\ell-1)}(\mathcal{B}), \{ X_{j}^{(\ell-1)}(\mathcal{B}): v_j \in \mathcal{N}_{i}^{+}\}, 
    & \{ X_{j}^{(\ell-1)}(\mathcal{U}): v_j \in \mathcal{N}_{i}^{-} \} \bigl)  \Big\}   \bigg) \\
    X_{i}^{(\ell)}(\mathcal{U}) & = \varphi \bigg( \Big\{ \bigl( X_{i}^{(\ell-1)}(\mathcal{U}), \{ X_{j}^{(\ell-1)}(\mathcal{U}): v_j \in \mathcal{N}_{i}^{+}\}, 
    & \{ X_{j}^{(\ell-1)}(\mathcal{B}): v_j \in \mathcal{N}_{i}^{-} \} \bigl)  \Big\}   \bigg) 
\end{aligned}
\end{equation}
where $\varphi$ is an injective function. 
\end{definition}

The extended WL-test above is defined with a similar aggregation and update process as a SGNN and thus can be used to capture the expressibility of the SGNN.

\begin{definition}
 A {\em (rooted) $k$-hop ego-tree} is a tree built from a root node $v_i$ (level-0) in $\mathcal{G}$ inductively for $k$ levels: From any node $v_j$ at level $\ell\geq 0$, create a copy of each neighbor $v_p\in \mathcal{N}_j$ at level $\ell+1$ and connect $v_j$ and $v_p$ with a new tree edge whose sign is $\sigma(e_{j,p})$.   
 \end{definition}

\begin{definition}[ego-tree isomorphism]\label{def:Signed Graph Isomorphism}
Two signed ego-tree $\tau_1$ and $\tau_2$ are considered isomorphic, denoted by $\tau_1 \cong \tau_2$, if there exists a bijective mapping $\psi\colon \mathcal{V}_{\tau_1} \rightarrow \mathcal{V}_{\tau_2}$ such that for every pair of vertices $v_i, v_j \in \mathcal{V}_{\tau_1}$, an edge $e_{ij} \in \mathcal{E}_{\tau_1}$ exists if and only if $e_{\psi(i),\psi(j)} \in \mathcal{E}_{\tau_2}$, and the sign of the corresponding edge satisfy $A_{ij} = A_{\psi(i),\psi(j)}$.
\end{definition}

\begin{theorem}\label{isomophic_egotree}
Suppose two ego-trees $\tau_1$ and $\tau_2$ are  isomorphic. An SGNN $\mathcal{A}$ applied to $\tau_1$ and $\tau_2$ will produce the same node embedding for the roots of these ego-trees.  
\end{theorem}

\begin{proof}
\label{prf:isomophic_egotree}
Suppose ego-tree $\tau_1$ and $\tau_2$ are two isomorphic signed graphs. After $\ell$ iterations, we have $\mathcal{A}(root(\tau_1)) \neq \mathcal{A}(root(\tau_2))$, where $root(\tau)$ represents the root of $\tau$. As $\tau_1$ and $\tau_2$ are isomorphic, they have the same extended WL node labels for iteration $\ell$ for any $\ell = 0,\dots,k-1$, i.e., $X_1^{(\ell)}(\mathcal{B}) = X_2^{(\ell)}(\mathcal{B})$ and $X_1^{(\ell)}(\mathcal{U}) = X_2^{(\ell)}(\mathcal{U})$, as well as the same collection of neighbor labels, i.e., \vspace*{-0.1cm}
\begin{equation}
    \scriptsize
    \begin{aligned}
    \bigg(X_1^{(\ell)}&(\mathcal{B}),\{X_j^{(\ell)}(\mathcal{B}):v_j \in \mathcal{N}_1^{+} \}, \{X_1^{(\ell)}(\mathcal{U}): v_j \in \mathcal{N}_1^{-} \} \bigg) = \\
    &\bigg(X_2^{(\ell)}(\mathcal{B}),\{X_j^{(\ell)}(\mathcal{B}):v_j \in \mathcal{N}_2^{+} \}, \{X_2^{(\ell)}(\mathcal{U}): v_j \in \mathcal{N}_2^{-} \} \bigg) \\
    \bigg(X_1^{(\ell)}&(\mathcal{U}),\{X_j^{(\ell)}(\mathcal{U}):v_j \in \mathcal{N}_1^{+} \}, \{X_1^{(\ell)}(\mathcal{B}): v_j \in \mathcal{N}_1^{-} \} \bigg) = \\
    &\bigg(X_2^{(\ell)}(\mathcal{U}),\{X_j^{(\ell)}(\mathcal{U}):v_j \in \mathcal{N}_2^{+} \}, \{X_2^{(\ell)}(\mathcal{B}): v_j \in \mathcal{N}_2^{-} \} \bigg)
    \end{aligned}
\end{equation}
Otherwise, the extended WL test should have obtained different node labels for $\tau_1$ and $\tau_2$ at iteration $\ell+1$. As the $\psi$ is an injective function, the extended WL test always relabels different extended multisets into different labels. As the SGNN and extended WL test follow the similar aggregation and rebel process, if $X_1^{(\ell)} = X_2^{(\ell)}$, we can have $h_1^{(\ell)} = h_2^{(\ell)}$. Thus, $\mathcal{A}(root(\tau_1)) = \mathcal{A}(root(\tau_2))$, we have reached a contradiction.
\end{proof}

\begin{figure*}
    \centering    \includegraphics[width=1.0\columnwidth]{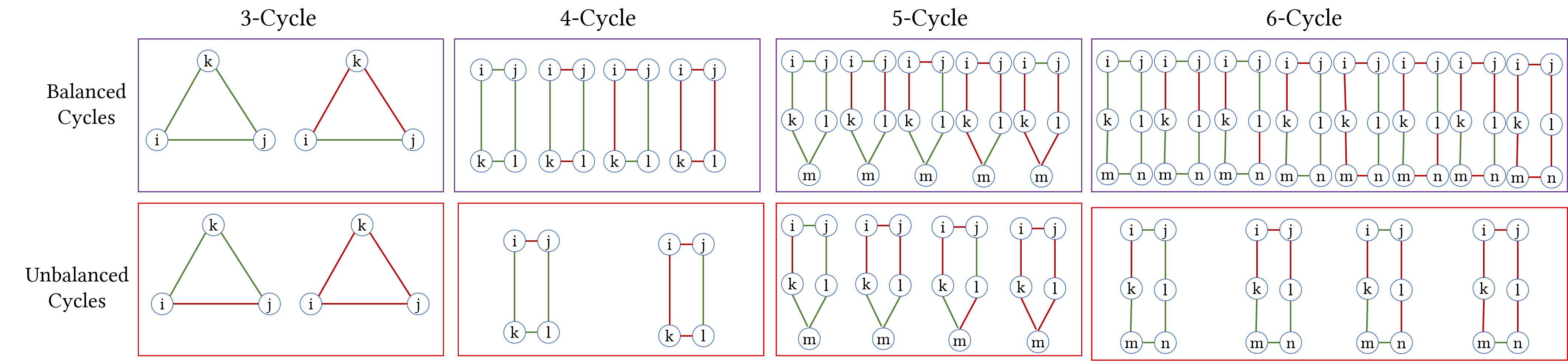}
    \caption{Isomorphism types of balanced (unbalanced) 3-cycle, 4-cycle, 5-cycle and 6-cycle. Green and red lines represent + and - edges, resp.}
    \label{fig:cycles}
\end{figure*}

We now turn our attention to cycles. According to small-world experiment \footnote{\url{https://en.wikipedia.org/wiki/Six_degrees_of_separation\#cite_note-1}}, we only consider cycles with a maximum of 6 nodes. Fig.\ref{fig:cycles} shows isomorphism types of balanced (unbalanced) 3-cycle, 4-cycle, 5-cycle and 6-cycle.

\begin{theorem}\label{SGNN_limitation}
An SGNN cannot learn adequate representations for edges from unbalanced cycles.
\end{theorem}
\begin{figure}
    \centering    \includegraphics[width=0.7\columnwidth]{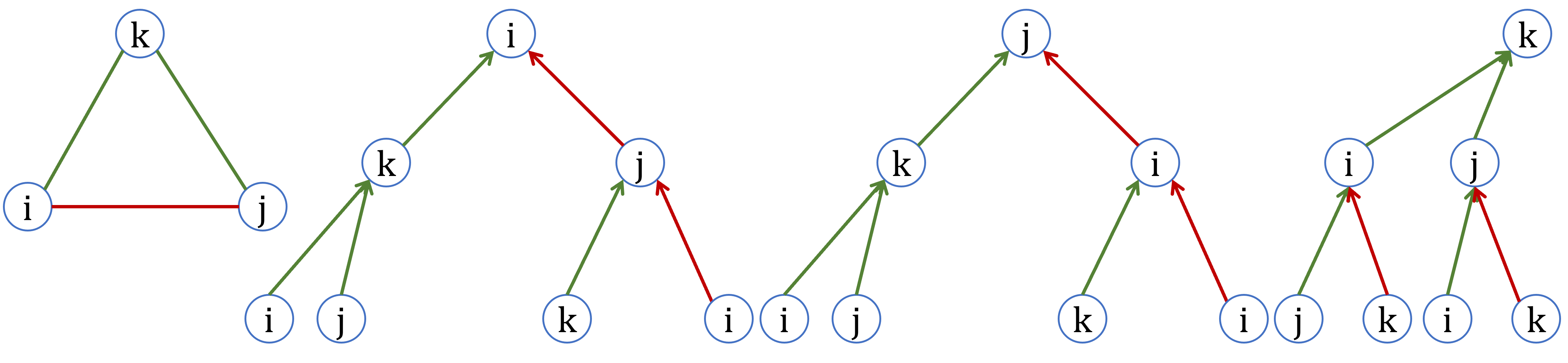}
    \caption{One unbalanced situation of cycle-3. Green and red lines represent + and - edges, resp.}
    \label{fig:situation (c)}
\end{figure}

\begin{figure}
    \centering    \includegraphics[width=0.7\columnwidth]{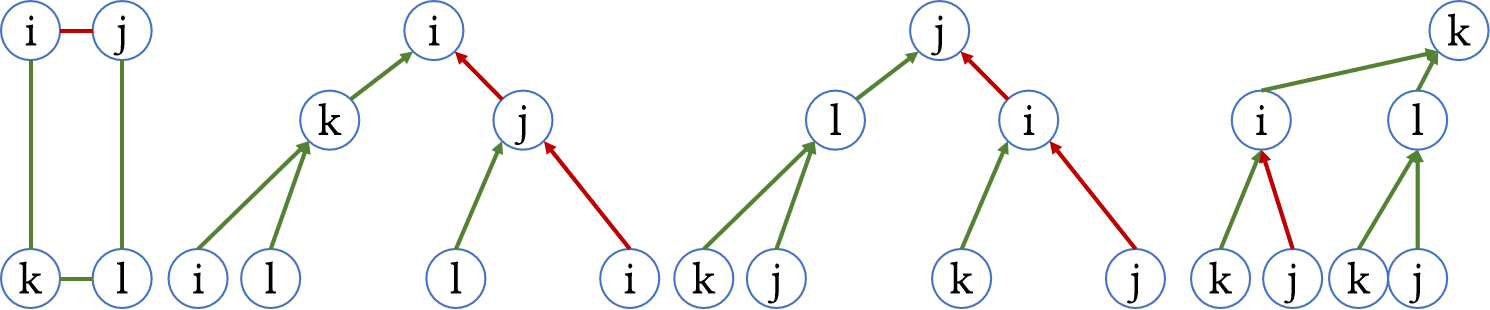}
    \caption{One unbalanced situation of cycle-4. Green and red lines represent + and - edges, resp.}
    \label{fig:cycle-4}
\end{figure}

\begin{proof}\label{prf:SGNN_limitation}
We consider one unbalanced situation of 3-cycle as shown in Figure \ref{fig:situation (c)}. In this scenario, we construct the 2-hop ego-trees ($\tau_i, \tau_j, \tau_k$) of nodes $v_i$, $v_j$, and $v_k$ as depicted in Figure~\ref{fig:situation (c)}. In constructing ego-trees, positive neighbors are positioned on the left side, while negative neighbors are on the right. It is evident that $\tau_{i}$ and $\tau_{j}$ exhibit isomorphism. Based
on the conclusion in \cite{xu2018powerful,zhang2023rsgnn}, they will be projected to the same embeddings. Conversely, as $\tau_i$ and $\tau_k$ are not isomorphic, they will be mapped to distinct embeddings. Thus, we deduce $dist(H_i,H_j) \leq dist(H_i, H_k)$, where $dist$ represents a distance metric, indicating that nodes connected by negative edges have closer representations than those connected by positive edges. By Definition \ref{def:adequate}, the learned representations $H_i, H_j, H_k$ are deemed inadequate for $v_i$, $v_j$, and $v_k$. Consequently, the representations of edges $e_{ij}$, $e_{ik}$, and $e_{jk}$ are also considered inadequate (see Def. \ref{def:adequate_edges}). Intuitively, during the training process of SGNN, node $v_i$ tends to pull node $v_j$ closer through edges $e_{ik}$ and $e_{kj}$, while simultaneously pushing $v_j$ away through edge $e_{ij}$. The conflicting structural information makes it hard for SGNN to learn an adequate spatial position for the three nodes.

\begin{figure}
    \centering
    \includegraphics[width=0.7\columnwidth]{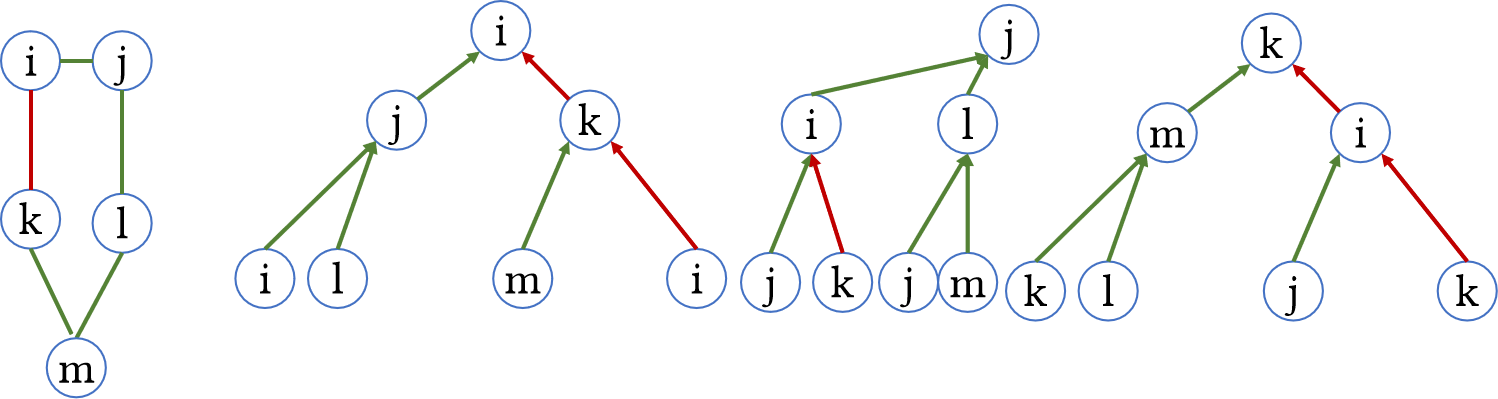}
    \caption{One unbalanced situation of cycle-5. Green and red lines represent + and - edges, resp.}
    \label{fig:cycle-5}
\end{figure}

\begin{figure}
    \centering
    \includegraphics[width=0.7\columnwidth]{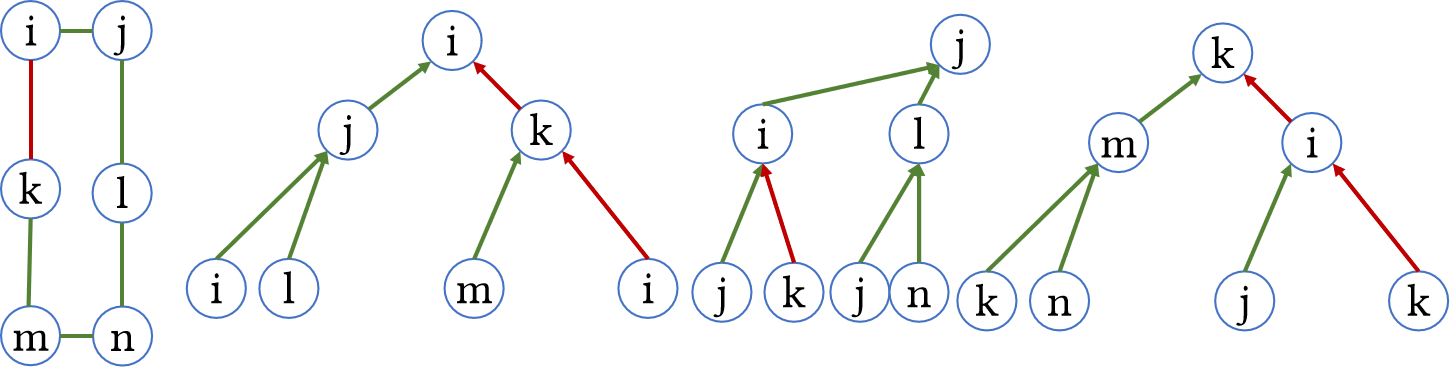}
    \caption{One unbalanced situation of cycle-6. Green and red lines represent + and - edges, resp.}
    \label{fig:cycle-6}
\end{figure}

We next consider one unbalanced situation of 4-cycle as shown in Figure \ref{fig:cycle-4}. In this scenario, we construct the 2-hop ego-trees of nodes ($\tau_i, \tau_j, \tau_k$)  $v_i$, $v_j$, and $v_k$ as depicted in Figure~\ref{fig:cycle-4}. Positive neighbors are positioned on the left side, while negative neighbors are on the right. Similar to the above scenario, it is evident that $\tau_{i}$ and $\tau_{j}$ exhibit isomorphism and $\tau_i$ and $\tau_k$ are not isomorphic. Thus, we get $dist(H_i,H_j) \leq dist(H_i, H_k)$. But, node $v_i$ is connected to $v_j$ with negative edge and is connected to $v_k$ with positive edge. By Definition \ref{def:adequate} and \ref{def:adequate_edges}, the representations of edges $e_{ij}$, $e_{ik}$, and $e_{jk}$ are considered inadequate.

Next, we consider one unbalanced situation of 5-cycle as shown in Figure \ref{fig:cycle-5}. In this situation, we construct the 2-hop ego-trees of nodes $v_i$, $v_j$, and $v_k$ as depicted in Figure~\ref{fig:cycle-5}. Positive neighbors are positioned on the left side, while negative neighbors are on the right. It is evident to find that $\tau_{i}$ and $\tau_{k}$ exhibit isomorphism and $\tau_i$ and $\tau_j$ are not isomorphic. Thus, we get $dist(H_i,H_k) \leq dist(H_i, H_j)$. But node $v_i$ is connected to $v_k$ with negative edge and is connected to $v_j$ with positive edge. By Definition \ref{def:adequate} and \ref{def:adequate_edges}, the representations of edges $e_{ij}$, $e_{ik}$, and $e_{jk}$ are considered inadequate.

Next, we consider one unbalanced situation of 6-cycle as shown in Figure \ref{fig:cycle-6}. When considering only 2-hop ego-tree, this situation is similar to the 5-cycle case mentioned above, so the proof is omitted. Considering that the majority of SGNN models only utilize information from two-hop neighbors \cite{derr2018signed,li2020learning}, it is reasonable for us not to consider the 3-hop ego-tree structure.
\end{proof}

\subsection{Triangle-based Difficulty measurer}
In this subsection, we describe the process of measuring the difficulty scores of training samples. Based on the above analysis, we conclude that SGNNs struggle to learn adequate representations for edges in unbalanced cycles. Thus, unbalanced cycles are difficult structures for SGNNs to learn. According to Table \ref{tab:n-cycle}, unbalanced triangles are more common than other unbalanced cycles, making them a greater challenge for model training. To improve efficiency, we
only consider the impact of unbalanced triangles on SGNN models. Intuitively, as shown in Figure \ref{fig:difficulty}, if an edge belongs to an unbalanced triangle, its difficulty score should be higher than those that do not. Firstly, we define local balance degree: 
\begin{table*}[]
\centering
\caption{The statistic of n-cycles ($n = \{3,4,5,6\}$) in six real-world datasets (see Sec. \ref{sec:settings}). \# n-cycles refers to the number of n-cycles, \# B(U)-cycles refers to the number of balanced (unbalanced) n-cycles.}
\resizebox{\linewidth}{!}{
\begin{tabular}{lcccccccccccccccccccccccc}
\hline
            & \multicolumn{4}{c}{Epinions}  & \multicolumn{4}{c}{Slashdot} & \multicolumn{4}{c}{Bitcoin-alpha} & \multicolumn{4}{c}{Bitcoin-OTC} & \multicolumn{4}{c}{WikiElec} & \multicolumn{4}{c}{WikiRfa} \\ \hline
n-cycle     & 3      & 4     & 5     & 6    & 3      & 4     & 5    & 6    & 3       & 4      & 5      & 6     & 3      & 4      & 5     & 6     & 3      & 4     & 5    & 6    & 3     & 4     & 5    & 6    \\ \hline
\# n-cycles & 159034 & 22224 & 13057 & 7073 & 24505  & 8813  & 3162 & 3757 & 3198    & 802    & 499    & 451   & 5027   & 1465   & 746   & 527   & 25610  & 15179 & 6194 & 4728 & 35024 & 11172 & 6949 & 5067 \\ \hline
\# B-cycles & 145559 & 19576 & 10635 & 5456 & 21436  & 7913  & 1909 & 2209 & 2793    & 625    & 389    & 355   & 4525   & 1175   & 548   & 356   & 20596  & 11605 & 4600 & 3527 & 26057 & 7624  & 4144 & 3008 \\ \hline
\# U-cycles & 13475  & 2648  & 2422  & 1617 & 3069   & 900   & 1253 & 1548 & 405     & 177    & 110    & 96    & 502    & 290    & 198   & 171   & 5014   & 3574  & 1594 & 1201 & 8967  & 3548  & 2805 & 2059 \\ \hline
\end{tabular}
}
\label{tab:n-cycle}
\end{table*}

\begin{figure}
    \centering
    \includegraphics[width=0.7\columnwidth]{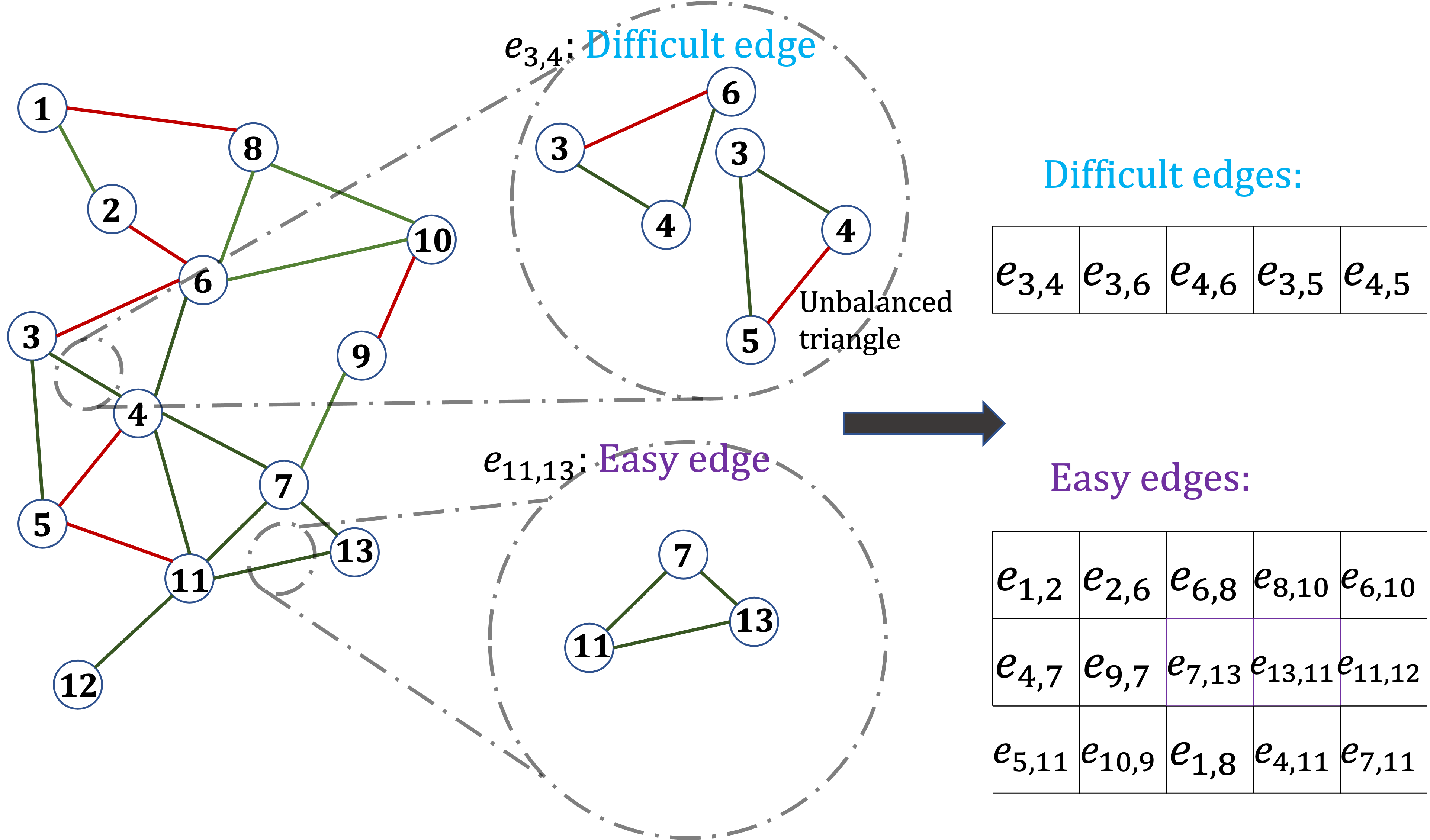}
    \caption{Illustration of node difficulty, where green lines represent positive edges and red lines represent negative edges.}
    \label{fig:difficulty}\vspace*{-0.2cm}
\end{figure}

\begin{definition}[Local Balance Degree]\label{local_balance_degree}
For edge $e_{ij}$, the local balance degree is defined by:
\begin{equation}
\label{eq:triangle_index}
D_3(e_{ij})=\frac{|\mathcal{O}_3^{+}(e_{ij})|}{|\mathcal{O}_3(e_{ij})|}
\end{equation}
where $\mathcal{O}_3^{+}(e_{ij})$ represents the set of balanced triangles containing edge $e_{ij}$, $\mathcal{O}_3(e_{ij})$ represents the set of triangles containing edge $e_{ij}$. $|\cdot|$ represents the set cardinal number.
\end{definition}

Based on this definition, for edge $v_{ij}$, if all triangles containing it are balanced, the $D_3(e_{ij})$ is $1$, if all of the triangles containing it are unbalanced, the $D_3(e_{ij})$ is $0$. For those edges that do not belong to any triangles, we set their local balance degree to $1$.
After obtaining the local balance degree for each edge, we can calculate the difficulty score of edge $e_{ij}$ as below:
\begin{equation}\label{eq:score}
    Score(e_{ij}) =  1 - \frac{|\mathcal{O}_3^{+}(e_{ij})|}{|\mathcal{O}_3(e_{ij})|}
\end{equation}

\subsection{Training Scheduler}
After measuring the difficulty score of each sample (i.e., edge) in the training set, we use a curriculum-based training strategy to train a better SGNN model, as shown in Figure \ref{fig:framework}(3). We follow similar methods in \cite{wei2022clnode} to generate the easy-to-difficult curriculum. More specifically, we first sort the training set $\mathcal{E}$ in ascending order of difficulty scores. Then, a pacing function $g(t)$ is used to place these edges to different training epochs from easy to difficult, where $t$ refers to $t$-th epoch. In this paper, we consider three pacing functions, i.e., linear, root, and geometric. The linear function increases the difficulty of training samples at a uniform rate; the root function introduces more difficult samples in fewer epochs, while the geometric function trains for a greater number of epochs on the subset of easy edges before introducing difficult edges. These three functions are defined:{ \footnotesize $g(t)=\min \left(1, \lambda_0+\left(1-\lambda_0\right) * \frac{t}{T}\right)$ (linear)}, {\footnotesize $g(t)=\min \left(1, \sqrt{\lambda_0^2+\left(1-\lambda_0^2\right) * \frac{t}{T}}\right)$ (root)}, {\footnotesize $g(t)=\min \left(1, \sqrt{\lambda_0^2+\left(1-\lambda_0^2\right) * \frac{t}{T}} \right)$ (geometric)}. $\lambda_0$ denotes the initial proportion of the available easiest examples and $T$ denotes the training epoch when $g(t)$ reaches 1. 

The process of CSG is detailed in Algorithm \ref{algo:csg}. The CSG method is highly efficient, with the majority of computational cost stemming from Equation \ref{eq:triangle_index}, which has a time complexity of $O(r)$, where $r$ represents the maximum number of neighbors for a single node. Calculating $\mathcal{O}_3(e_{ij})$ and $\mathcal{O}_3^{+}(e_{ij})$ is equivalent to identifying the common neighbors of nodes $i$ and $j$, which requires searching through two ordered neighbor lists and takes $O(r)$ time.
\begin{algorithm}
\footnotesize
\caption{CSG Training Algorithm\label{algo:csg}}
\begin{algorithmic}
\STATE \textbf{Data:} A signed graph $\mathcal{G} = (\mathcal{V}, \mathcal{E})$, the SGNN model $f$, the pacing function $g(t)$
\STATE \textbf{Result:} SGNN parameters $\theta$
\STATE Initialize SGNN parameter $\theta$\
\FOR{$e_{ij} \in \mathcal{E}$}
    \STATE $Score(e_{ij}) \leftarrow$ Eq.\ref{eq:score}
\ENDFOR
\STATE Sort $\mathcal{E}$ according to difficulty in ascending order\;
\STATE Let $t=1$\
\WHILE{\textit{Stopping condition is not met}}
\STATE $\lambda_t \leftarrow g(t)$ \;
\STATE $\mathcal{E}_t \leftarrow \mathcal{E}\left[0, \ldots, \lambda_t \cdot |\mathcal{E}|\right]$\;
\STATE Use $f$ to predict the labels $\hat{\sigma(\mathcal{E}_t)}$\;
\STATE Calculate cross-entropy loss $\mathcal{L}$ on $\{\hat{\sigma(\mathcal{E}_t)}, \sigma(\mathcal{E}_t)\}$\;
\STATE Back-propagation on $f$ for minimizing $\mathcal{L}$\;
\STATE $t \leftarrow t+1$
\ENDWHILE
\STATE \textbf{return} $\theta$
\end{algorithmic}
\end{algorithm}

\section{Experiments}
In this section, we initiate our evaluation by examining the improvements facilitated by CSG when compared to various SGNN models for link sign prediction. 
Following this, we examine how model performance varies with different training dataset proportions. Then, we analyze the performance differences between hard and easy samples under the CSG training framework. Lastly, we perform ablation studies to assess the impact of different pacing functions and explore CSG's parameter sensitivity.

\subsection{Experimental Settings}\label{sec:settings}
We perform experiments on six real-world datasets: Bitcoin-OTC, Bitcoin-Alpha, WikiElec, WikiRfa, Epinions, and Slashdot. Key statistical information is provided in Table \ref{tab:datasets}. Since these datasets have no node features, we randomly generate a 64-dimensional vector for each node as the initial features. In the following, we introduce datasets briefly.

\begin{table}[]
\footnotesize
\centering 
\caption{The statistics of datasets.}
\resizebox{0.7\textwidth}{!}{ 
\begin{tabular}{cccc}
\hline
Dataset       & \# Links & \# Positive Links & \# Negative Links \\ \hline
Bitcoin-OTC   & 35,592    & 32,029             & 3,563              \\
Bitcoin-Alpha & 24.186    & 22,650             & 1,536              \\
WikiElec      & 103,689   & 81,345             & 22,344             \\
WikiRfa       & 170,335   & 133,330            & 37,005             \\
Epinions      & 840,799   & 717,129            & 123,670            \\
Slashdot      & 549,202   & 425,072            & 124,130            \\ \hline
\end{tabular}}
\label{tab:datasets}
\end{table}

\textbf{Bitcoin-OTC} \cite{kumar2016edge} and \textbf{Bitcoin-Alpha} are two datasets extracted from Bitcoin trading platforms. Due to the fact Bitcoin accounts are anonymous, people give trust or not-trust tags to others to enhance security.

\textbf{WikiElec} \cite{leskovec2010signed,leskovec2010predicting} is a voting network in which users can choose to trust or distrust other users in administer elections. \textbf{WikiRfa} \cite{west2014exploiting} is a more recent version of WikiElec.

\textbf{Epinions} \cite{leskovec2010signed} is a consumer review site with trust or distrust relationships between users.

\textbf{Slashdot} \cite{leskovec2010signed} is a technology-related news website in which users can tag each other as friends (trust) or enemies (distrust).

\begin{table}[]
\centering
\caption{Statistics of triangles in each experiments. TR refers to training ratio. B (U) refers to the number of Balanced (Unbalanced) triangles. R (\%) refers to the ratio of B/U.}
\scriptsize
\begin{tabular}{cccccccccc}
\hline
TR                & \multicolumn{3}{c}{Bitcoin-otc} & \multicolumn{3}{c}{Bitcoin-Alpha} & \multicolumn{3}{c}{WikiElec} \\ \hline
                     & B         & U        & R        & B          & U         & R        & B         & U       & R      \\ \hline
40\%                 & 1690      & 99       & 17.1     & 1819       & 292       & 6.2      & 2396      & 1678    & 1.4    \\
50\%                 & 1977      & 134      & 14.8     & 1904       & 271       & 7.0      & 3752      & 2267    & 1.7    \\
60\%                 & 2678      & 211      & 12.7     & 2218       & 260       & 8.5      & 5701      & 2493    & 2.3    \\
70\%                 & 3378      & 291      & 11.6     & 2689       & 565       & 4.8      & 8446      & 3980    & 2.1    \\
80\%                 & 3486      & 307      & 11.4     & 2696       & 370       & 7.3      & 10853     & 3741    & 2.9     \\ \hline
TR                  & \multicolumn{3}{c}{WikiRfa}     & \multicolumn{3}{c}{Epinions}      & \multicolumn{3}{c}{Slashdot} \\ \hline
\multicolumn{1}{l}{} & B         & U        & R        & B          & U         & R        & B         & U       & R      \\ \hline
40\%                 & 14284     & 5720     & 2.5      & 62649      & 6825      & 9.2      & 12645     & 2115    & 6.0    \\
50\%                 & 14302     & 4736     & 3.0      & 77537      & 13538     & 5.7      & 15626     & 2535    & 6.2    \\
60\%                 & 13747     & 4504     & 3.1      & 85150      & 13360     & 6.4      & 14872     & 2992    & 5.0    \\
70\%                 & 15965     & 6227     & 2.6      & 91532      & 13546     & 6.8      & 15128     & 2697    & 5.6    \\
80\%                 & 23607     & 7276     & 3.2      & 106055     & 7776      & 13.6     & 17080     & 3106    & 5.5    \\ \hline
\end{tabular}
\label{tab:stats_tri}
\end{table}

Further statistics regarding balanced and unbalanced triangles are provided in Table \ref{tab:stats_tri}, encompassing training ratios spanning from 40\% to 80\%. Importantly, the statistics showcases a consistent ratio of both balanced and unbalanced triangles across all proportional edge selections for the training set, indicating a stable performance regardless of the training set size. The experiments were performed on a Windows machine with eight 3.8GHz AMD cores and a 80GB A100 GPU.

We use five popular SGNNs as the backbone models, namely SGCN \cite{derr2018signed}, SNEA \cite{li2020learning}, SDGNN \cite{huang2021sdgnn}, SGCL \cite{shu2021sgcl} and GS-GNN \cite{liu2021signed}, which are representative methods of SGNNs. 
With regard to the hyper-parameters in the baselines, to facilitate fair comparison, we employ the same backbone SGNN of CSG as the baselines. We set $\lambda_0$ to 0.25, $T$ to 20 and use the linear pacing function by default.

\subsection{Experiment Results}
\begin{table*}[]
\scriptsize
\centering
\caption{Link sign prediction results (average $\pm$ standard deviation) with AUC (\%) on six benchmark datasets.}
\resizebox{\textwidth}{!}{ 
\begin{tabular}{cccccccc}
\hline
                        & Method    & Bitcoin-OTC                  & Bitcoin-Alpha              & WikiElec & WikiRfa & Epinions & Slashdot \\ \hline
\multirow{3}{*}{SGCN}   & Original  & 82.5 $\pm$ \scriptsize{4.3} & 79.2 $\pm$ \scriptsize{4.4} & 65.7 $\pm$ \scriptsize{8.1}  & 66.1 $\pm$ \scriptsize{9.1}    &  72.5 $\pm$ \scriptsize{4.7} & 58.6 $\pm$ \scriptsize{4.9}        \\
                        & +CSG      & 85.3 $\pm$ \scriptsize{1.6} & 85.1  $\pm$ \scriptsize{1.5}  & 78.1 $\pm$ \scriptsize{1.0} & 76.6 $\pm$ \scriptsize{0.7} & 80.3 $\pm$ \scriptsize{1.5}         &  72.5 $\pm$ \scriptsize{0.4}        \\
                        & (Improv.) &  3.4\%        &  7.4\%          &  18.9\%    &  15.9\%   &  10.7\%    &  23.7\%    \\ \hline
\multirow{3}{*}{SNEA}   & Original  &  82.8 $\pm$ \scriptsize{3.9} & 81.2 $\pm$ \scriptsize{4.1}  & 69.3 $\pm$ \scriptsize{6.5} & 69.8 $\pm$ \scriptsize{5.2}  & 77.3 $\pm$ \scriptsize{3.1} & 66.3 $\pm$ \scriptsize{4.2}          \\
                        & +CSG      & 86.3 $\pm$ \scriptsize{1.3} & 87.1 $\pm$ \scriptsize{1.3}   & 79.3 $\pm$ \scriptsize{1.1}  & 78.2 $\pm$ \scriptsize{1.0}  & 81.7 $\pm$ \scriptsize{0.8}    & 75.1 $\pm$ \scriptsize{0.7}           \\
                        & (Improv.) &  4.2\%        &  7.2\%       & 14.4\%        & 12.0\%        &  5.7\%        &  13.3\%        \\ \hline
\multirow{3}{*}{SDGNN}  & Original  & 85.3 $\pm$ \scriptsize{5.3} &  82.2 $\pm$ \scriptsize{4.7} & 73.3 $\pm$ \scriptsize{6.1} & 76.8 $\pm$ \scriptsize{4.3}  & 81.3 $\pm$ \scriptsize{4.8} & 73.3 $\pm$ \scriptsize{4.4}         \\
                        & +CSG      & 88.1 $\pm$ \scriptsize{1.5}  &  87.5 $\pm$ \scriptsize{2.0}  & 80.7 $\pm$ \scriptsize{1.6} & 81.0 $\pm$ \scriptsize{1.0}    & 85.5 $\pm$ \scriptsize{0.7} & 77.3 $\pm$ \scriptsize{1.7}          \\
                        & (Improv.) &  3.3\%      &     6.4\%     & 10.1\%   & 5.5\%   & 5.2\%    & 5.5\%           \\ \hline
\multirow{3}{*}{SGCL}   & Original  & 88.2 $\pm$ \scriptsize{6.2} & 85.4 $\pm$ \scriptsize{5.2}  & 80.4 $\pm$ \scriptsize{4.1} & 82.1 $\pm$ \scriptsize{3.8} & 86.4 $\pm$ \scriptsize{5.1}& 82.3 $\pm$ \scriptsize{5.1}         \\
                        & +CSG      & 90.3 $\pm$ \scriptsize{1.2} & 89.2 $\pm$ \scriptsize{1.4} & 85.2 $\pm$ \scriptsize{1.8}  & 86.2 $\pm$ \scriptsize{2.0}& 88.7 $\pm$ \scriptsize{1.3} & 86.1 $\pm$ \scriptsize{1.1}          \\
                        & (Improv.) &  2.4\%      & 4.4\%         &  6.0\%   & 5.0\%   & 2.7\%    & 4.6\%          \\ \hline
\multirow{3}{*}{GS-GNN} & Original  & 89.1 $\pm$ \scriptsize{4.3}  & 87.3 $\pm$ \scriptsize{4.9}  & 81.3 $\pm$ \scriptsize{5.0}  & 80.5 $\pm$ \scriptsize{4.1}& 88.3 $\pm$ \scriptsize{3.5}& 90.7 $\pm$ \scriptsize{4.4}         \\
                        & +CSG      &  94.1 $\pm$ \scriptsize{1.1}  & 92.6 $\pm$ \scriptsize{1.9}  & 86.7 $\pm$ \scriptsize{2.1}  & 87.2 $\pm$ \scriptsize{1.1} & 92.6 $\pm$ \scriptsize{2.1} & 94.2 $\pm$ \scriptsize{1.0}          \\
                        & (Improv.) & 5.6\%  & 6.1\% & 6.6\%          &  8.3\%        &  4.9\%         &   3.9\%       \\ \hline
\end{tabular}
}
\label{tab:auc}
\end{table*}

As per \cite{huang2021signed}, we use $10\%$ test, $5\%$ validation, and $85\%$ training data across datasets. Five runs yield average AUC and F1-binary scores and deviations in Table \ref{tab:auc} and Table \ref{tab:performance_f1}.

Our conclusions from the results are as follows: 1. CSG effectively enhances the performance of five prominent SGNN models in the context of link sign prediction. 2. Integration with CSG reduces the standard deviation of SGNNs' performance significantly, indicating a reduction in the inherent randomness of the backbone SGNN models. 
3. It is worth highlighting that the impact of CSG's performance improvements varies across datasets, with Bitcoin-OTC, Bitcoin-Alpha, and Epinions showing relatively modest gains compared to WikiElec, WikiRfa, and Slashdot. This discrepancy can be attributed to the lower ratio of unbalanced triangles in the former group, as evidence from the data in Table \ref{tab:stats_tri}, which reduces the influence of the training scheduler and restricts the potential for performance enhancement.
\begin{table*}[]
\centering
\footnotesize
\caption{Link Sign Prediction Results with F1 (\%) on six benchmark datasets}
\resizebox{\textwidth}{!}{ 
\begin{tabular}{cccccccc}
\hline
                        & Method    & Bitcoin-OTC                  & Bitcoin-Alpha              & WikiElec & WikiRfa & Epinions & Slashdot \\ \hline
\multirow{3}{*}{SGCN}   & Original  & 92.1 $\pm$ \scriptsize{1.9} & 92.9 $\pm$ \scriptsize{0.8} & 86.0 $\pm$ \scriptsize{2.8}  & 84.2 $\pm$ \scriptsize{2.3}    &  91.5 $\pm$ \scriptsize{0.9} & 83.6 $\pm$ \scriptsize{2.9}        \\
                        & +CSG      & 93.9 $\pm$ \scriptsize{0.8} & 93.9  $\pm$ \scriptsize{0.6}  & 87.1 $\pm$ \scriptsize{0.6} & 86.0 $\pm$ \scriptsize{0.4} & 92.7 $\pm$ \scriptsize{0.4}         &  84.6 $\pm$ \scriptsize{0.3}        \\
                        & (Improv.) &  2.0\%        & 1.1\%          &  1.3\%    &  2.1\%   &  1.3\%    &  1.2\%    \\ \hline
\multirow{3}{*}{SNEA}   & Original  & 92.5 $\pm$ \scriptsize{2.2} & 92.8 $\pm$ \scriptsize{0.9}   & 86.3 $\pm$ \scriptsize{2.5} & 84.2 $\pm$ \scriptsize{2.3}   & 92.1 $\pm$ \scriptsize{1.2} & 84.0 $\pm$ \scriptsize{2.3}          \\
                        & +CSG      & 94.1 $\pm$ \scriptsize{0.3} & 94.3  $\pm$ \scriptsize{0.3}  & 87.1 $\pm$ \scriptsize{0.6}  &  86.6 $\pm$ \scriptsize{0.3} & 93.4$\pm$ \scriptsize{0.6}     &   86.1 $\pm$ \scriptsize{0.4}      \\
                        & (Improv.) &  1.6\%     &  1.6\%         &   0.9\%   & 2.9\%   & 1.4\%   &  2.5\%         \\ \hline
\multirow{3}{*}{SDGNN}  & Original  &  91.3 $\pm$ \scriptsize{2.1} & 93.1 $\pm$ \scriptsize{1.0}  & 87.7 $\pm$ \scriptsize{2.2} & 85.3 $\pm$ \scriptsize{2.5} &92.7 $\pm$ \scriptsize{1.1}& 85.8 $\pm$ \scriptsize{1.9}         \\
                        & +CSG      & 94.3 $\pm$ \scriptsize{0.5}  & 93.8  $\pm$ \scriptsize{0.3}  & 89.2 $\pm$ \scriptsize{0.5} & 87.1 $\pm$ \scriptsize{1.0}  & 93.3 $\pm$ \scriptsize{0.4}    & 88.5 $\pm$ \scriptsize{0.2}         \\
                        & (Improv.) &  3.3\%      & 0.8\%         &  1.7\%   & 2.1\%   & 0.6\%    &  3.1\%        \\ \hline
\multirow{3}{*}{SGCL}   & Original  &  92.5 $\pm$ \scriptsize{1.5}  & 92.5 $\pm$ \scriptsize{1.1} & 89.6 $\pm$ \scriptsize{1.2} & 88.1 $\pm$ \scriptsize{1.6}  & 95.3 $\pm$ \scriptsize{1.3}  & 90.1 $\pm$ \scriptsize{1.1}         \\
                        & +CSG      & 94.2 $\pm$ \scriptsize{0.2}  &  93.6 $\pm$ \scriptsize{0.2}   & 91.1 $\pm$ \scriptsize{0.4}  & 92.6 $\pm$ \scriptsize{0.7}&  96.7 $\pm$ \scriptsize{0.3}  &92.5 $\pm$ \scriptsize{0.3}          \\
                        & (Improv.) &  1.8\%      &  1.2\%        & 1.7\%     & 5.1\%  &   1.5\%   &   2.7\%       \\ \hline
\multirow{3}{*}{GS-GNN} & Original  & 92.5 $\pm$ \scriptsize{1.7}  &  93.5 $\pm$ \scriptsize{2.1}  & 90.3 $\pm$ \scriptsize{1.5}  & 92.1 $\pm$ \scriptsize{1.6}& 94.1 $\pm$ \scriptsize{1.8}  &  89.8 $\pm$ \scriptsize{2.1}         \\
                        & +CSG      & 94.2 $\pm$ \scriptsize{0.7}  & 94.8 $\pm$ \scriptsize{0.3}  & 92.7 $\pm$ \scriptsize{0.8}  & 94.2 $\pm$ \scriptsize{0.3}   & 95.3 $\pm$ \scriptsize{0.5} & 91.9 $\pm$ \scriptsize{1.0}         \\
                        & (Improv.) & 1.8\%   &  1.4\%       &  2.7\%    &  2.2\%     &  1.3\%         &   2.3\%       \\ \hline
\end{tabular}
}
\label{tab:performance_f1}
\end{table*}
We further verify the effectiveness of CSG on more incomplete graphs, says 40\% - 70\% edges as training samples, 5\% edges as validation set and the remaining as test set. we use SGCN \cite{derr2018signed} as backbone(original) model, the results are shown in Table \ref{tab:incomplete}. From the results we conclude: 1. With a decrease in the training ratio, the model's performance indeed exhibits a decline. This can be attributed to the reduced amount of information available for the model, consequently leading to a decrease in its predictive capability. 2. The stabilizing effect of CSG's improvement is evident. In Table 4, we suggest that adjusting training proportions is unlikely to greatly alter the balance between balanced and unbalanced triangles, thus contributing to the consistent enhancement brought by CSG.
\begin{table*}[]
\setlength{\abovecaptionskip}{0.1cm}
\caption{Experimental Performance (AUC, average $\pm$ standard deviation) on Training ratio from 40\% - 70\%, TR = training ratio.}\label{tab:incomplete}
\centering
\footnotesize
\resizebox{\textwidth}{!}{ 
\begin{tabular}{llllllll}
\hline
TR                    & Data    & Bitcoin-OTC & Bitcoin-Alpha & WikiElec & WikiRfa & Epinions & Slashdot \\ \hline
\multirow{3}{*}{70\%} & Original  & 80.3 $\pm$ 4.3  & 77.1 $\pm$ 5.4 &  63.2 $\pm$ 7.7 & 63.5 $\pm$ 5.5 & 71.3 $\pm$ 5.2& 59.3 $\pm$ 4.8\\
                      & +CSG      & 85.0 $\pm$ 1.5  & 84.8 $\pm$ 1.3  &  75.5 $\pm$ 1.3 & 74.3 $\pm$ 1.3  & 79.8 $\pm$ 0.8 & 72.0 $\pm$ 1.7 \\
                      & (Improv.) & 5.9\%  & 10.0\%              &  19.4\%        &  17.0\%       &  11.9\%  & 21.4\%     \\ \hline
\multirow{3}{*}{60\%} & Original  & 79.6 $\pm$ 2.6 & 76.5 $\pm$ 4.2  & 61.7 $\pm$ 5.6  & 62.1 $\pm$ 3.7  & 70.5 $\pm$ 4.3 & 57.5 $\pm$ 6.3     \\
                      & +CSG      & 83.8 $\pm$ 1.3 & 83.5 $\pm$ 1.1  & 73.1 $\pm$ 1.6  & 73.1 $\pm$ 1.1  & 78.5 $\pm$ 1.1 &70.7$\pm$ 1.5    \\
                      & (Improv.) &  5.3\%   & 9.2\%                & 18.4\%         & 17.7\%   & 11.3\%     & 23.0\%  \\ \hline
\multirow{3}{*}{50\%} & Original  & 76.3 $\pm$ 6.2 & 74.2 $\pm$ 4.4  & 60.3 $\pm$ 4.9 & 63.3 $\pm$ 5.2  & 68.9 $\pm$ 5.8 & 57.1 $\pm$ 5.4 \\
                      & +CSG      & 83.4 $\pm$ 1.8 & 83.1 $\pm$ 1.3 & 71.5 $\pm$ 1.2  & 72.8 $\pm$ 0.9 & 78.3 $\pm$ 0.9  & 70.2$\pm$ 1.3  \\
                      & (Improv.) &  9.3\%           &  12.0\%      &  18.6\%        & 15.0\%        & 13.6\%  &   22.9\%       \\ \hline
\multirow{3}{*}{40\%} & Original  &  78.3 $\pm$ 3.1 & 75.3 $\pm$ 5.1 & 61.7 $\pm$ 6.1 & 64.1 $\pm$ 4.2 & 69.6 $\pm$ 6.1 & 57.3 $\pm$ 7.1          \\
                      & +CSG      &  82.0 $\pm$ 1.1 & 82.7 $\pm$ 1.7 & 70.1 $\pm$ 1.9 & 72.5 $\pm$ 1.2 & 77.1 $\pm$ 1.6 & 69.8$\pm$ 1.1         \\
                      & (Improv.) &  4.7\%     &  9.8\%             & 13.6\%         & 13.1\%        & 10.8\% & 21.8\%          \\ \hline
\end{tabular}
}
\end{table*}

Prior experiments validated CSG's efficiency. Next, we'll analyze improved prediction performance via CSG for specific sample types. 
We categorize test edges into easy and hard edges, based on unbalanced triangle membership. The link sign prediction performance for both types is shown in Table \ref{tab:easy}. CSG enhances the original model for \textit{both} cases, particularly benefiting easy edges. This agrees with our expectation that delaying unbalanced triangle training yields greater stability for straightforward edges. Based on the preceding theoretical analysis, we can deduce that SGNN struggles to learn appropriate representations from hard edges. When both easy and hard edges are simultaneously incorporated into the training process, the presence of hard edges might disrupt the model's ability to learn adequate representations from the easy edges. However, by prioritizing the consideration of easy edges, we effectively sidestep the interference caused by these hard edges, resulting in a significant improvement in learning information from the easy edges. Consequently, this approach also leads to a marginal enhancement in the performance of hard edges.
\begin{table*}[]
\centering
\footnotesize
\caption{Link Sign Prediction Performance (AUC, average $\pm$ standard deviation) for Easy Links and Hard Links.}
\resizebox{\textwidth}{!}{ 
\begin{tabular}{cccccccc}
\hline
                            &           & Bitcoin-OTC & Bitcoin-Alpha & WikiElec & WikiRfa & Epinions & Slashdot \\ \hline
Easy Links                  & Backbone  & 84.3 $\pm$ 5.1       & 80.9 $\pm$ 4.8         & 67.1 $\pm$ 8.4    & 68.3 $\pm$ 9.3   & 74.1 $\pm$ 4.1     & 60.1 $\pm$ 5.2     \\
                            & +CSG      & 87.2 $\pm$ 1.5       & 86.8 $\pm$ 0.8         & 79.9 $\pm$ 1.7    & 78.5 $\pm$ 1.2   & 82.3 $\pm$ 1.3    & 74.1 $\pm$ 1.0    \\
                            & (Improv.) & 3.4\%                & 7.3\%                  & 19.1\%            & 14.9\%           & 11.1\%   & 23.3\%   \\ \hline
\multirow{3}{*}{Hard Links} & Backbone  & 75.3 $\pm$ 4.5       & 74.1 $\pm$ 6.1         & 60.2 $\pm$ 7.2    & 61.3 $\pm$ 8.1   & 66.4 $\pm$ 4.4    & 51.2 $\pm$ 4.1     \\
                            & +CSG      & 76.1 $\pm$ 1.2       & 78.6 $\pm$ 1.9         & 64.5 $\pm$ 0.7     & 65.5 $\pm$ 0.7   & 68.3 $\pm$ 1.5    & 55.8 $\pm$ 1.5     \\
                            & (Improv.) & 1.1\%                & 6.1\%                  & 7.1\%   & 6.9\%            & 2.9\%    & 9.0\%    \\ \hline
\end{tabular}
}
\label{tab:easy}
\end{table*}

\subsection{Ablation Study}

\begin{table}[]
\centering
\footnotesize
\caption{Test AUC (\%) results (average $\pm$ standard deviation) on six benchmark datasets with different pacing functions.}
\begin{tabular}{lccc}
\hline
\multicolumn{1}{c}{} & linear & root & geometric \\ \hline
Bitcoin-OTC          & \textbf{85.3} $\pm$ \scriptsize{1.6}       & 85.2 $\pm$ \scriptsize{1.5}  & 85.0 $\pm$ \scriptsize{1.4}          \\
Bitcoin-Alpha        & 85.1  $\pm$ \scriptsize{1.5}       &  85.2 $\pm$ \scriptsize{1.2}     &  \textbf{85.6} $\pm$ \scriptsize{1.8}         \\
WikiElec             & 78.1 $\pm$ \scriptsize{1.0}       & 77.6 $\pm$ \scriptsize{0.6}     & \textbf{78.4} $\pm$ \scriptsize{1.2}          \\
WikiRfa              & \textbf{76.6} $\pm$ \scriptsize{0.7}       & 76.2 $\pm$ \scriptsize{0.8}     & 76.2 $\pm$ \scriptsize{0.6}         \\
Epinions             & 80.3 $\pm$ \scriptsize{1.5}       & \textbf{80.9} $\pm$ \scriptsize{0.6}     &  79.6 $\pm$ \scriptsize{1.1}        \\
Slashdot             & \textbf{72.5} $\pm$ \scriptsize{0.4}       & 71.5 $\pm$ \scriptsize{1.2}      & 71.1 $\pm$ \scriptsize{1.4}           \\ \hline
\end{tabular}
\label{tab:ablation}
\end{table}

\begin{table}[]
\centering
\footnotesize
\caption{Test F1 (\%) resultss (average $\pm$ standard deviation) on six benchmark datasets with different pacing functions.}
\begin{tabular}{lccc}
\hline
\multicolumn{1}{c}{} & linear & root & geometric \\ \hline
Bitcoin-OTC          & \textbf{93.9} $\pm$ \scriptsize{0.8}   & \textbf{93.9} $\pm$ \scriptsize{0.8}   &  93.8 $\pm$ \scriptsize{0.5}         \\
Bitcoin-Alpha        & 93.9  $\pm$ \scriptsize{0.6}  & \textbf{94.2} $\pm$ \scriptsize{0.2}    &   93.8 $\pm$ \scriptsize{0.3}        \\
WikiElec             & 87.1 $\pm$ \scriptsize{0.6}    & \textbf{87.3} $\pm$ \scriptsize{0.6}      & 86.3 $\pm$ \scriptsize{0.7}           \\
WikiRfa              & 86.0 $\pm$ \scriptsize{0.4}  &  \textbf{86.4} $\pm$ \scriptsize{0.7}     & 85.7 $\pm$ \scriptsize{0.8}          \\
Epinions             & 92.7 $\pm$ \scriptsize{0.4}  & \textbf{92.9} $\pm$ \scriptsize{0.4}     & 92.6 $\pm$ \scriptsize{0.4}          \\
Slashdot             & \textbf{84.6} $\pm$ \scriptsize{0.3}   & 84.5 $\pm$ \scriptsize{0.9}      & 84.5 $\pm$ \scriptsize{1.3}        \\ \hline
\end{tabular}
\label{tab:ablation_f1}
\end{table}

We test different pacing functions (linear, root, geometric) with SGCN as the backbone model. Table \ref{tab:ablation} shows a slight edge for linear pacing. Yet, in general, the pacing function choice minimally affects CSG's performance. The reason for this experimental outcome might lie in the fact that the real-world datasets are exceedingly sparse, resulting in a limited number of edges belonging to unbalanced triangles. Therefore, different pacing functions brings about negligible changes in the weights of training samples. A significant portion of hard edges is fed into the model in the final few rounds of training. More F1-binary experimental results refer to Table \ref{tab:ablation_f1}.  

\subsection{Parameter Sensitivity}
\begin{figure}
    \centering    \includegraphics[width=0.7\columnwidth]{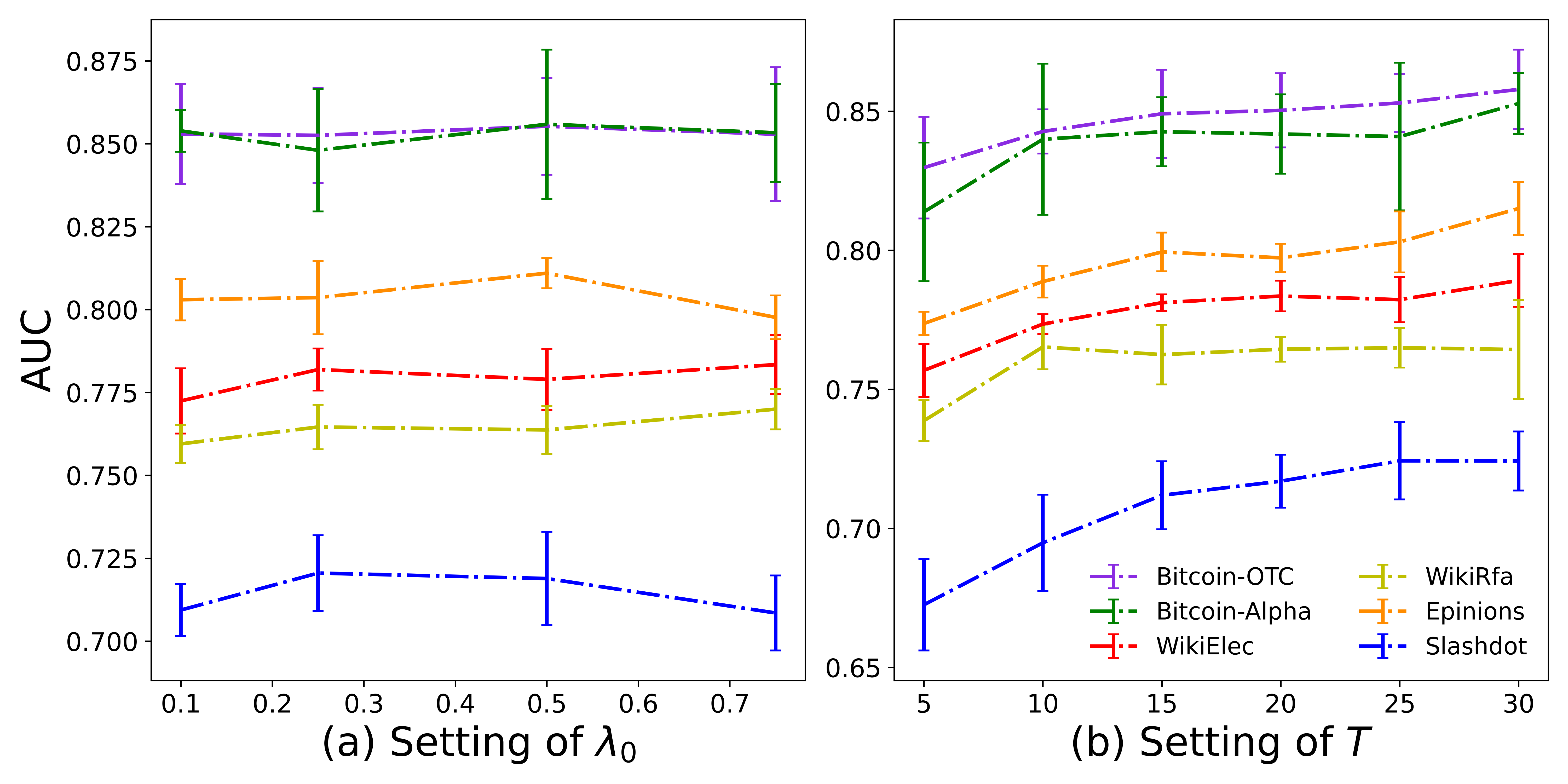}
    \caption{AUC result (average $\pm$ standard deviation) of CSG under different values of the hyper-parameters $\lambda_0$ and $T$}
    \label{fig:para}
\end{figure}
We examine how $\lambda_0$ and $T$ affect CSG performance. $\lambda_0$ sets initial training ratio, and $T$ controls difficult sample introduction speed. To explore the parameter sensitivity, we select $\lambda_0$ from $\{0.1, 0.25, 0.5, 0.75\}$ and $T$ from $\{5, 10, 15, 20, 25,30\}$, respectively. We use SGCN as the backbone and employ linear pacing function. The result in Figure \ref{fig:para} shows the following: (1) generally, with the $\lambda_0$ increasing, the AUC value first increases and then decreases. The performance is reasonable for most datasets when $\lambda_0\in [0.25, 0.5]$. A too-smaller $\lambda_0$ means fewer training samples are introduced in the initial stage of model training, which makes model incapable of learning efficiently. But as $\lambda_0$ increases, more edges with high difficult scores are used in initial training process which will degrade the model performance. (2) As $T$ increases, the model performance improves quickly. But this trend slows down as $T$ continues to increase. A too-small $T$ quickly introduces more difficult edges which can degrade the performance of backbone. A too-large $T$ makes SGNNs to be trained on the easy edges most of the time which causes a loss of the information and increases the training time.

\section{Conclusion}
We explore SGNN training, typically with randomly ordered samples, resulting in equal contributions. We propose SGNNs benefit from a curriculum approach similar to human learning, introducing CSG rooted in curriculum learning. While CSG doesn't address the representation limitations of SGNNs, it effectively alleviate the negative impact of hard samples and enhance the performance of backbone model. Extensive experiments on six benchmark datasets demonstrate CSG's versatility in boosting various SGNN models. Future promising directions include exploring theoretical foundations of graph curriculum learning and devising potent methods for diverse downstream tasks on signed graphs.

\section*{Acknowledgment}
This study was supported by the major project of Science and Technology Department of Xinjiang (2024A02002). This study was also supported by the National Natural Science Foundation of China (W2411020, 32170645). We thank the high-performance computing platform at the National Key Laboratory of Crop Genetic Improvement in Huazhong Agricultural University.

\bibliographystyle{elsarticle-num} 
\bibliography{references}






\end{document}